\documentclass[brazil,english, conference, american, harvard]{sbatex}
\usepackage[utf8]{inputenc}
\usepackage[T1]{fontenc}
\makeatletter
\def\verbatim@font{\normalfont\ttfamily\footnotesize}
\makeatother
\usepackage{amsmath}
\usepackage{graphicx}
\usepackage{float}
\usepackage{amssymb}
\usepackage{subfigure}
\usepackage{ae}
\usepackage{epstopdf}
\usepackage{algpseudocode}
\usepackage{algorithm}
\usepackage{color}
\usepackage[hidelinks]{hyperref}

\newcommand{\blfootnote}[1]{%
  \begingroup
  \renewcommand{\thefootnote}{}%
  \footnotetext{#1}%
  \addtocounter{footnote}{-1}%
  \endgroup
}

\newtheorem{lema}{Lemma}

\begin{document}

\title{AN OPTIMAL ALGORITHM FOR CHANGING FROM LONGITUDINAL TO LATITUDINAL FORMATION OF AUTONOMOUS AIRCRAFT SQUADRONS}

\author[1]{Paulo André S. Giacomin,}
          {pasgiacomin@uesc.br, }

\address{Universidade Estadual de Santa Cruz\\
Departamento de Ciências Exatas e Tecnológicas - Área de Informática\\
Ilhéus-BA, Brasil}

\address{Instituto Tecnológico de Aeronáutica\\
Departamento de Sistemas e Controle - Divisão de Engenharia Eletrônica \\
São José dos Campos-SP, Brasil}

\author[2]{Elder M. Hemerly}
          {hemerly@ita.br}

\twocolumn[

\maketitle

\selectlanguage{english}

\begin{abstract}
This work presents an algorithm for changing from longitudinal to 
latitudinal formation of autonomous aircraft squadrons. The maneuvers are defined 
dynamically by using a predefined set of 3D basic maneuvers. 
This  formation change  is necessary when the squadron has to perform
tasks  demanding  both formations, such as landing,  lift-off,   obstacle avoidance and georeferencing.
The time complexity analysis of the transformation algorithm 
reveals that its efficiency is optimal. 
Simulations  show that the  formation change  is  done  without collision,  and   the proof of  correctness  ensures its latitudinal formation features.
\end{abstract}


\selectlanguage{brazil}

\begin{abstract}
Este trabalho apresenta um algoritmo de mudança de formação em longitude
para formação latitudinal de esquadrilhas de aeronaves autônomas. As manobras são definidas
dinamicamente utilizando-se um conjunto pré-definido de manobras 3D básicas. 
Esta mudança de formação é necessária quando a esquadrilha tem que desenvolver tarefas que
demandam ambas as formações, tais como 
aterrissagem, decolagem, desvio de obstáculos e georreferenciamento.
A análise 
de complexidade de tempo do algoritmo de transformação revela que sua eficiência é ótima.
As simulações mostram que a mudança de formação é feita sem colisão,  
e a prova de correção  assegura suas características de formação latitudinal.  
\end{abstract}

]
\selectlanguage{english}

\section{Introduction}

\blfootnote{ Published in the Proceeding Series of the Brazilian Society of Computational and Applied Mathematics (DINCON 2013); the scientific content remains unchanged. Licence: CC BY-NC-ND. \href{https://doi.org/10.5281/zenodo.21473681}{DOI: 10.5281/zenodo.21473681.}  }

Recently, it has been possible to see a growing interest in the development of autonomous 
aircraft that can cooperate with several organizations in the solution of 
public security problems. Police and army are some examples. 
This happens because the autonomous agents can deal with
unhealthy or dangerous problems, like monitoring of areas with armed conflict, 
violence monitoring, fires, deforestation monitoring and inspection of 
nuclear areas, without to expose humans to the risks. 

When several agents are used in the solution of these problems, some advantages arise:
a) it is possible to make a better use of sensors, since they can be shared by the network, b) 
distributed systems are usually more robust than the centralized ones. 
c) autonomous aircraft typically
fly at low heights, hence good quality images can be captured, and d) when autonomous aircraft squadrons are used in georeferencing, 
the visual field of the cameras increases,since the
 images captured by each aircraft can be mosaicked. 

The longitudinal formation presented in Figure \ref{fig:longitudinal-formation}
is necessary when the squadron is landing and doing  lift-off,   and
can also be used for collision avoidance: if the first aircraft succeeds in avoiding the collision, all other
aircraft in the squadron can also avoid the obstacle by using the same behavior. 

The latitudinal formation shown in Figure \ref{fig:latitudinal-formation} 
is an attractive formation to deal with the georeferencing problem,
since better area coverage can be achieved. Therefore, if the same squadron needs to perform
obstacle avoidance, landing,  lift-off  and georeferencing, it will eventually be necessary
for the squadron to change between its longitudinal and latitudinal formations. 

\begin{figure}[h]
 \centering
  \includegraphics[bb=0 0 148 178,scale=1]{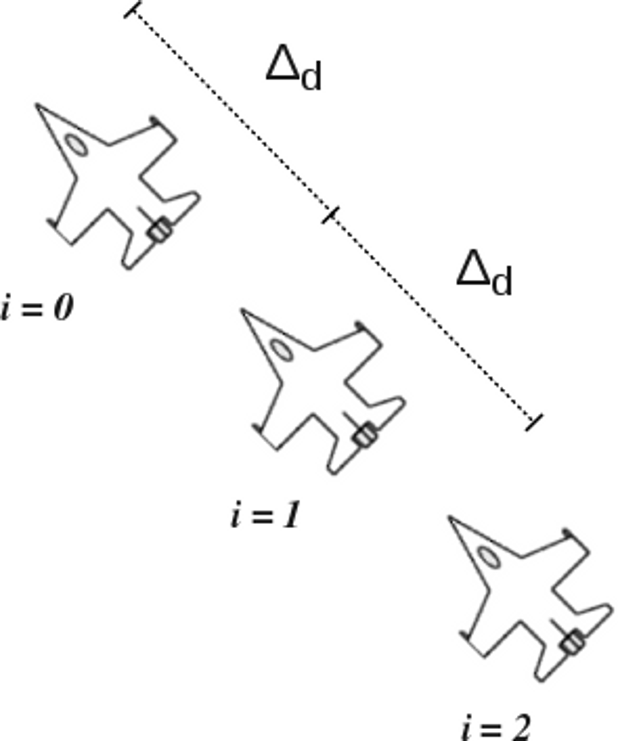}
 \caption{The longitudinal formation can be used for collision avoidance, landing and  lift-off.  }
 \label{fig:longitudinal-formation}
\end{figure}

Therefore, the problem considered in this work is: assuming that 
there are $n$ aircraft in longitudinal formation, equally spaced by 
$\Delta_d$ meters, we want to develop an algorithm that changes 
the squadron to the latitudinal formation, where the aircraft 
will also be equally spaced by $\Delta_d$ meters, without collision among them.

\begin{figure}[h]
 \centering
 \includegraphics[bb=0 0 163 163,scale=1]{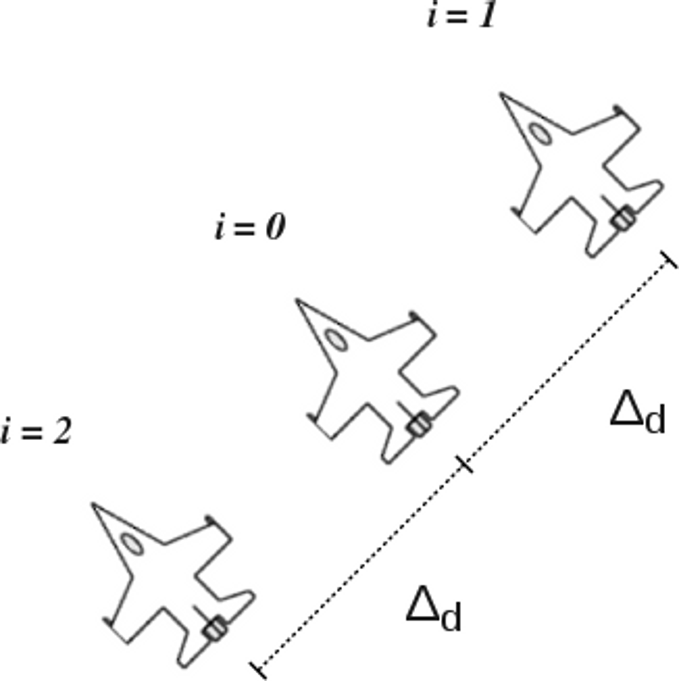}
 \caption{It is possible to achieve good area coverage with the latitudinal formation.}
 \label{fig:latitudinal-formation}
\end{figure}

This problem is different of the one considered in \cite{Giacomin-sbai}, 
since the one treated here deals with transition from longitudinal to
latitudinal formation, but the one presented in \cite{Giacomin-sbai} 
deals with changing from latitudinal to longitudinal formation.
Therefore, although apparently similar, the algorithms, sequence of maneuvers, 
simulation results and theoretical analysis are quite different.  

The problem considered here involves control, trajectory generation and aircraft formation. 
Control techniques, like the reinforcement learning \cite{Santos} do not
focus on the changing of formation.
Trajectory generation  done  by using optimization algorithms are
usually found in the literature. Some examples are the works of 
\cite{Cheng} and \cite{Xu}, but works like these do not 
focus on well defined geometric formations. Other techniques are also used in the trajectory
definition, like the geometric moments, controlled via nonlinear gradient \cite{Morbidi}, 
the modern matrix analysis used by \cite{Coker}, the combination of the hybrid navigation 
architecture with the local obstacle avoidance methodology and with the model
predictive control \cite{Jansen}, the navigation functions \cite{Roussos}, and the
variation of rapidly-exploring random trees \cite{Neto}. But
they also do not consider the changing between well defined geometric formations during 
the flight. The formation problem is treated by the nonlinear model predictive control approach
of \cite{Chao} and by the leader-follower approach of \cite{You} and \cite{Gu}, 
but changing between well defined geometric formations is not considered.   

The formation reconfiguration is 
studied by \cite{Knoll}, where the aircraft move their position 
inside a formation, and the same formation is considered before and after 
the reconfiguration. The autonomous decision-making architecture \cite{Venkataramanan}
also considers this problem, but neither \cite{Knoll} nor \cite{Venkataramanan} consider 
the transition between different formations. 
After exhaustive literature searching,  no algorithm was found  dealing  with the problem considered here.
Thus, the main contributions of this work are:

\begin{enumerate}

\item An algorithm, adapted from \cite{Giacomin-sbai}, for changing from the longitudinal to the latitudinal
formation of the squadron. The time complexity analysis of the proposed algorithm demonstrates  that  its efficiency is optimal.

\item  The simulation results  considering  a case study  in  which the aircraft do not collide, and a proof of  correctness  of the proposed algorithm  that  ensures its latitudinal formation features.

\end{enumerate}

These contributions are described in the next sections, 
namely: \ref{sec:methodology} - Methodology, \ref{sec:simulation} - Simulation Results, 
\ref{sec:theoretical-analysis} - Theoretical Analysis and \ref{sec:conclusion} - Conclusion. 

\section{Methodology}
\label{sec:methodology}

A set of maneuvers is specified  so  that the proposed algorithm
can use it to create references to be followed by the aircraft.

\subsection{The Maneuvers Scheme}

Two 3D basic maneuvers are employed by the proposed algorithm, 
as described in \cite{Giacomin-cobem}.
Here they are called C and FW to implement the  \textit{to-turn}  and  \textit{go-forward}  maneuvers, respectively.
The  \textit{to-turn}  interface is C($P$, $\varphi$, $\beta$, $\theta$, $\alpha$, $vel$, $r$, $T$),
and the  \textit{go-forward}  interface is FW($P$, $\varphi$, $\beta$, $vel$, $d$, $T$),
where $P$ is the initial position $(x_0, y_0, h_0)$, $T$ is a vector time to be filled 
with time intervals, $\varphi$ is the angle between the maneuvers and the latitudinal axis 
in counterclockwise direction, $\theta$ is the spin angle around the longitudinal axis, $\beta$
is the angle between the maneuvers and the longitudinal-latitudinal plane, $\alpha$
is the angle of the arc of the  \textit{to-turn}  maneuver with radius $r$, $vel$ is the
velocity of the maneuvers and $d$ is the distance of the  \textit{go-forward}  maneuver.
The angles $\varphi$, $\beta$ and $\theta$ are applied firstly by using the 
coordinate system base $(0, 0, 0)$, and after the maneuver is translated to the  base $(x_0, y_0, h_0)$.

The C and FW algorithms are used by the 
transition algorithm described at Subsection \ref{subsection:the-transition-algorithm}. 

\subsection{The Transition Algorithm}
\label{subsection:the-transition-algorithm}

The transition algorithm is shown in Algorithm \ref{algo:fltl}. It is called 
FLOTLA due to the initial characters that describe its function:
From Longitudinal To Latitudinal formation. It basically performs 
the maneuvers presented in Figure \ref{fig:geometria-trajetoria}, and the equivalent mirrored ones,  with  each aircraft moving to the left and to the right of the squadron, respectively.

\begin{algorithm}
  
  \caption{The transition algorithm}\label{algo:fltl}
  
  \begin{algorithmic}[1]  
   \Procedure{FLOTLA}{$i$, $\Delta_d$, $vel$, $r$, $\varphi$, $\beta$, $P$}

    \State $T \gets \varnothing$ 
    \State $Ref \gets \varnothing$

    \State $side \gets direction(i, N)$

    \State $k \gets 0$

    \While {$k < 4$}

      \If{$|Ref| \neq 0$} \label{linha:inicio_encaixe}
	\State $P_{0} \gets penultimate(Ref)$
	\State $D \gets distance(P_1, P_0)$
	\State $\varphi_2 \gets sin^{-1}(\Delta_y / D)$
	\State $\beta_2 \gets sin^{-1}(\Delta_h / D)$ 

      \Else
	\State $\varphi_2 = \varphi$
        \State $\beta_2 = \beta$
      \EndIf \label{linha:final_encaixe}

      \If{$k = 0$} 
	\If{$side = LEFT$} \label{linha:inicio_primeira_curva}
	  \State $a \gets C(P_1, \varphi_2, \beta_2, 0, \frac{\pi}{4}, vel, r, T)$
	\Else
	  \State $a \gets C(P_1, \varphi_2, \beta_2, \pi, \frac{\pi}{4}, vel, r, T)$
	\EndIf \label{linha:final_primeira_curva}
      \ElsIf{$k = 1$}
	\If{$even(N)$}
	  \State $i_2 \gets \lfloor(N - i - 1) / 2\rfloor$ \label{linha:calculo_delta}
	  \State $\Delta \gets \Delta_d / 2 + i_2 \cdot \Delta_d$
	\Else
	    \If{$i \neq 0 $} \label{linha:calculo_delta_odd_1}
	      \State $\Delta \gets \lfloor(i+1)/2\rfloor \cdot \Delta_d$
	    \Else
	      \State $\Delta \gets 0$
	    \EndIf \label{linha:calculo_delta_odd_2}
	\EndIf

	\State $\Delta_2 \gets \sqrt{2} \cdot (\Delta - 2 \cdot r \cdot (1 - cos(\pi / 4)))$
	\State $a \gets FW(P_1, \varphi_2, \beta_2, vel, \Delta_2, T)$ \label{linha:vai_para_fora}
      \ElsIf{$k = 2$}

	\If{$side = LEFT$} \label{linha:inicio_segunda_curva}
	  \State $\varphi_3 \gets - \varphi_2$
	  \State $a \gets C(P_1, \varphi_3, \beta_2, \pi, \frac{\pi}{4}, vel, r, T)$
	\Else
	  \State $a \gets C(P_1, \varphi_2, \beta_2, 0, \frac{\pi}{4}, vel, r, T)$
	\EndIf \label{linha:final_segunda_curva}

      \Else
	\State $\Delta_m \gets advancing(i, \Delta_d) + k_2$ \label{linha:deslocamento}
	\State $a \gets FW(P_1, \varphi_2, \beta_2, vel, \Delta_m, T)$ \label{linha:manobra_final}
      \EndIf

    \State $Ref \gets Ref \cup a$
    \State $P_1 \gets last(Ref)$ \Comment does not modify Ref
    \State $k \gets k + 1$
    \EndWhile

    \State \textbf{return} $Ref$ 
   \EndProcedure
  \end{algorithmic} 
\end{algorithm}

\begin{algorithm}
  
  \caption{The advancing algorithm}\label{algo:advancing}
  
  \begin{algorithmic}[1]  
   \Procedure{advancing}{$i$, $\Delta_d$}

    \If{$odd(N)$}
      \If{$odd(i)$}
	\State $\Delta_m \gets \Delta_d + 3 \cdot \lfloor i / 2\rfloor \cdot \Delta_d$
      \Else
	\State $\Delta_m \gets 3 \cdot \lfloor i / 2\rfloor \cdot \Delta_d$
      \EndIf
    \Else
      \If{$odd(i)$}
	\State $\Delta_m \gets \Delta_d + 3 \cdot \lfloor (i-1) / 2\rfloor \cdot \Delta_d$
      \Else
	\State $\Delta_m \gets 3 \cdot \lfloor i / 2\rfloor \cdot \Delta_d$
      \EndIf
    \EndIf

    \State \textbf{return} $\Delta_m$ 
   \EndProcedure
  \end{algorithmic} 
\end{algorithm}

\begin{figure}[h]
 \centering
 \includegraphics[bb=0 0 156 184,scale=1]{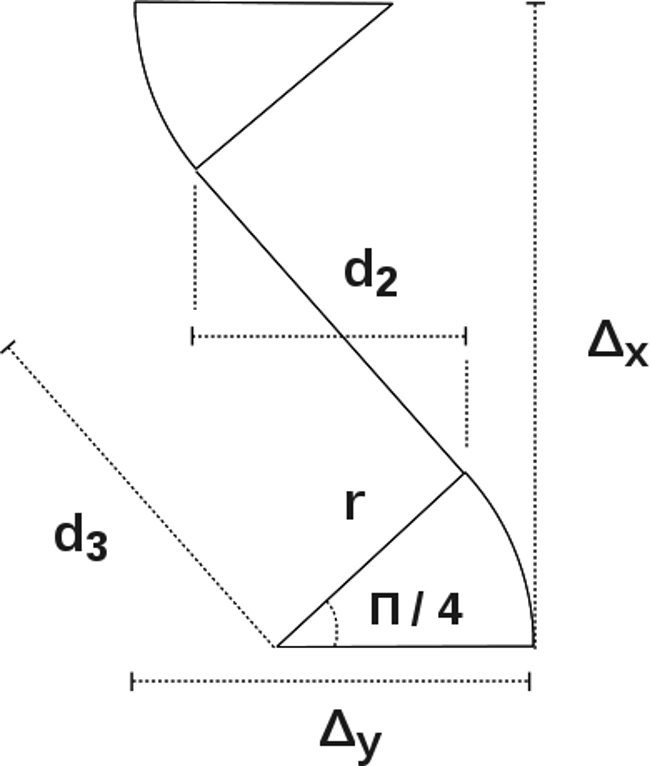}
 \caption{An aircraft goes out from the middle line
when it assumes the latitudinal formation. }
 \label{fig:geometria-trajetoria}
\end{figure}

The interface of the transition algorithm is FLOTLA($i$, $\Delta_d$, $vel$, $r$, $\varphi$, $\beta$, $P$), 
where $P$ is the aircraft initial position $(x_0, y_0, h_0)$,
$\varphi$ and $\beta$ have the same meaning that in the previously described C and FW interfaces,
$r$ is the radius of the  \textit{to-turn}  maneuvers
used by the transition algorithm,
$vel$ is the aircraft airspeed, and 
$i$ is the aircraft index shown in Figures \ref{fig:longitudinal-formation} and \ref{fig:latitudinal-formation}.

Algorithm \ref{algo:fltl} is executed by each aircraft processor  in  parallel.
Each aircraft executes four maneuvers. The values that are determined between
Lines \ref{linha:inicio_encaixe} and \ref{linha:final_encaixe} are used to fit each maneuver with 
the next one. 

Each aircraft goes out from the middle line of the squadron between Lines 
\ref{linha:inicio_primeira_curva} and \ref{linha:final_primeira_curva}. At Line \ref{linha:vai_para_fora}, 
the aircraft increases the diagonal distance from the middle line of the squadron. Between Lines \ref{linha:inicio_segunda_curva}
and \ref{linha:final_segunda_curva}, the aircraft moves to  assume  the latitudinal position.
Finally, for all aircraft, except the first one,
Algorithm \ref{algo:advancing} calculates the forward displacement, called $\Delta_m$, at Line \ref{linha:deslocamento}.
 Afterwards,  the aircraft enter into the latitudinal formation at Line \ref{linha:manobra_final}.

The aircraft model presented  in  Subsection \ref{subsection:model} was used to test the
references created  by  Algorithm \ref{algo:fltl}.

\subsection{The Aircraft Model}
\label{subsection:model}

The simple and well-tested aircraft state space model, see \cite{Anderson} for details, is employed  and  is given by

\begin{equation}
 \label{eq:velocidade}
 \frac{d V}{d t} = g \cdot \left[ \frac{(T - D)}{W} - sin(\gamma) \right]
\end{equation} 
\begin{equation}
 \label{eq:gama}
 \frac{d \gamma}{d t} = \frac{g}{V} \cdot [ n \cdot cos(\mu) - cos(\gamma)]
\end{equation} 
\begin{equation}
 \label{eq:heading}
 \frac{d \chi}{d t} = \frac{g \cdot n \cdot sin(\mu)}{V \cdot cos(\gamma)}
\end{equation} 
\begin{equation}
 \label{eq:x}
 \frac{d x}{d t} = V \cdot cos(\gamma) \cdot cos(\chi)
\end{equation} 
\begin{equation}
 \label{eq:y}
 \frac{d y}{d t} = V \cdot cos(\gamma) \cdot sin(\chi)
\end{equation} 
\begin{equation}
 \label{eq:z}
 \frac{d h}{d t} = V \cdot sin(\gamma)
\end{equation} 
where the state variables are: 
the position variables (x, y, h), 
flight path heading ($\chi$),
flight path angle ($\gamma$), and airspeed (V).

Here, the references created by Algorithm \ref{algo:fltl} are submitted to a control scheme
that  employs  the above aircraft model. It is presented in \cite{Giacomin-cobem}.

\section{Simulation Results}
\label{sec:simulation}

 Algorithm  \ref{algo:fltl} is programmed in parallel by using
the GNU Message Passing Interface (MPI) Compiler and
the C++ programming language.
The software Gnuplot is used for plotting the graphics.   

The initial condition for all aircraft  is:  $\Delta_d = 18,300.00$ meters, assuming the longitudinal formation,
airspeed: $30.5$ m/s, and height: $3,050.00$ meters.
A radius of  $4,575.00$ meters is used by Algorithm \ref{algo:fltl} to create the references,
and different airspeeds are calculated for each aircraft  that  allow all aircraft to  assume the  latitudinal formation at the same time instant. The results are shown in Figures \ref{fig:squadron-references}, \ref{fig:squadron-trajectories},
\ref{fig:velocity-dinamic} and \ref{fig:heading}. 
The Runge-Kutta-4 algorithm is used in the simulation. 

\begin{figure}[h]
 \centering
  \includegraphics[bb=0 0 219 134,scale=1]{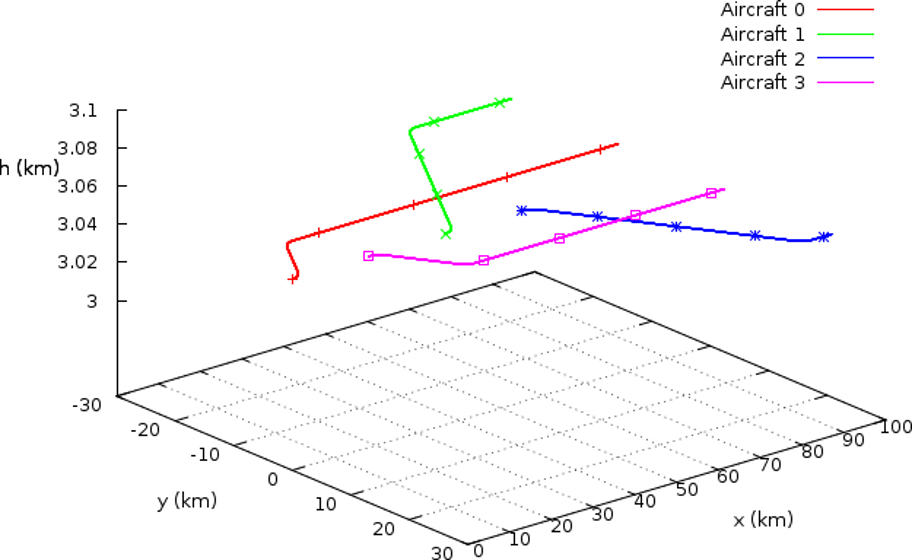}
 \caption{The aircraft references are marked for crossing analysis.}
 \label{fig:squadron-references}
\end{figure}

\begin{figure}[h]
 \centering
 \includegraphics[bb=0 0 219 134,scale=1]{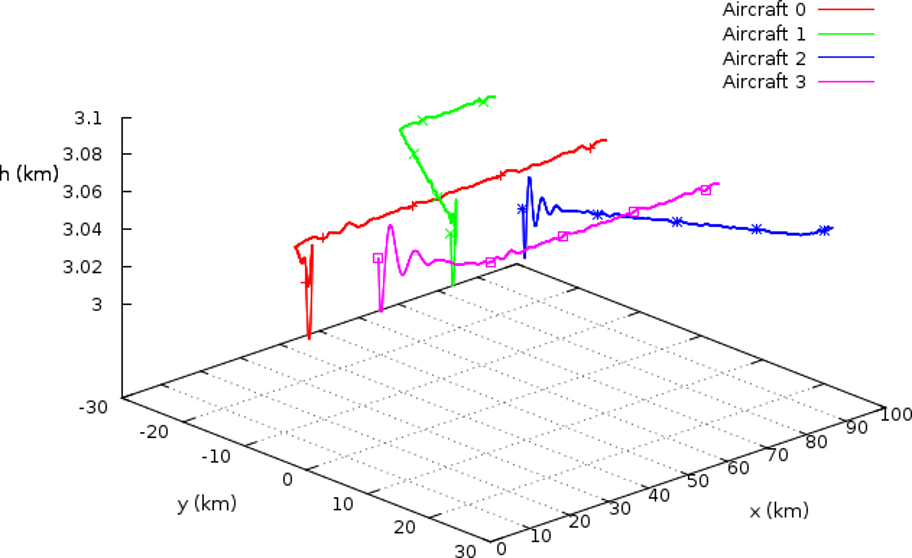}
 \caption{The aircraft trajectories succeed in following the references. }
 \label{fig:squadron-trajectories}
\end{figure}

The marks shown in Figures \ref{fig:squadron-references} and \ref{fig:squadron-trajectories} are
used to analyze the aircraft crossing. 
When the aircraft  numbered  zero and  three  are flying over the third mark, the aircraft  numbered  one and two are ahead. A simple verification algorithm did show that
the minimum distance between aircraft $0$ and $1$ and
between aircraft $2$ and $3$ was $18.3$ km and $16.8$ km, respectively. 
Therefore, there is no collision. 

The simulation is  done  by considering a noise of $0.5\%$ for airspeed, and a noise of
$0.25\%$ for heading and gamma angles. From Figures \ref{fig:squadron-references}, \ref{fig:squadron-trajectories},
\ref{fig:velocity-dinamic} and \ref{fig:heading},  it  is possible to conclude that Algorithm \ref{algo:fltl}
creates the references correctly and that the aircraft succeed in following the references.

\begin{figure}[h]
 \centering
 \includegraphics[bb=0 0 211 139,scale=1]{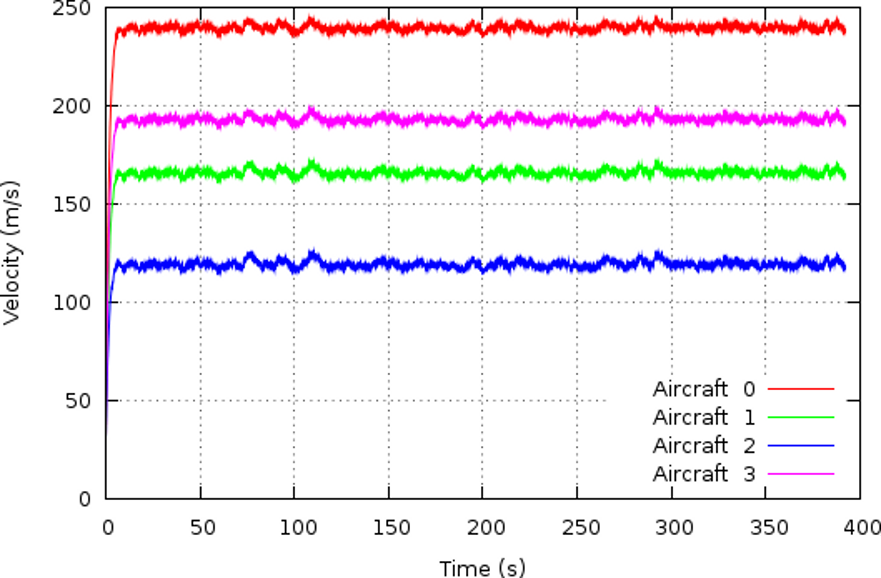}
 \caption{All aircraft velocities are  stable.  }
 \label{fig:velocity-dinamic}
\end{figure}

\begin{figure}[h]
 \centering
 \includegraphics[bb=0 0 211 139,scale=1]{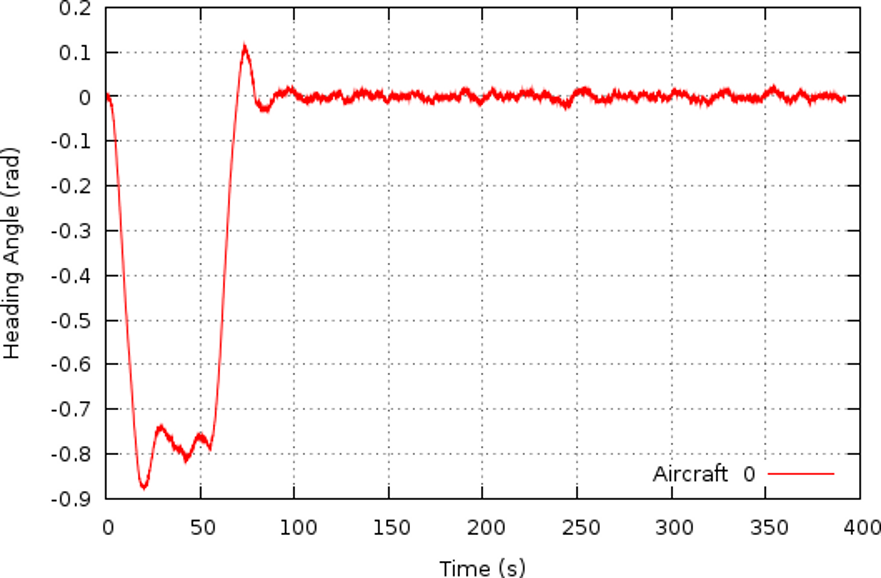}
 \caption{First aircraft heading angle switches between two values. }
 \label{fig:heading}
\end{figure}

It remains to prove the  correctness of  Algorithm \ref{algo:fltl} and to analyze its time complexity.
This is  done  in Section \ref{sec:theoretical-analysis}.

\section{Theoretical Analysis}
\label{sec:theoretical-analysis}

Firstly, the time complexity analysis of Algorithm \ref{algo:fltl}  is presented, followed by its proof of correctness.

\subsection{The Time Complexity Analysis}

The  \textit{to-turn}  and  \textit{go-forward}  functions are executed over a fixed
number of maneuvers, and all other functions used in Algorithm \ref{algo:fltl}
have  constant-time  complexity.
Let $Q = \{q_0, \cdots, q_{|Q|-1}\}$ be the 
set of maneuvers, and let $R = \{r_0, \cdots, r_{|Q|-1}\}$ be the
set of references, with $r_i = ref(q_i)$, and the function \textit{ref}
be implemented by using the functions C or FW in the Algorithm \ref{algo:fltl}. 
Then, $|r_i|$ is the  number of  references
for each maneuver $q_i$,
and $R$ is a set of sets. The time complexity of Algorithm \ref{algo:fltl} is

\[O(|r_0| + |r_1| + |r_2| + |r_3|)\]
\[m = |r_0| + \cdots + |r_3| \]
\[O(m)\]
where $m$ represents the total  number of  time steps for each aircraft trajectory.
\textit{Since this is also the lower bound for the problem,  one concludes that the  FLOTLA
algorithm has optimal time complexity \cite{Cormen}}. 

\subsection{The Proof of Correction}

The proof of correction takes into account that 
Figure \ref{fig:geometria-trajetoria} describes basically the maneuvers that are executed by
Algorithm \ref{algo:fltl}, except the last one,  which  is created by Algorithm \ref{algo:advancing}.

\begin{theorem}
\label{theorem:correcao}
Let assume a set of 
aircraft initially flying in longitudinal formation, with
each aircraft equally spaced by $\Delta_d$ from its neighbors. 
Then, if the FLOTLA algorithm is executed by each aircraft, 
the latitudinal formation is achieved, with spacing $\Delta_d$ between its neighbors. 
\end{theorem}
 
\begin{proof}
It is considered $n$ aircraft in the squadron.
Clearly, if $n$ is even, it can be concluded by 
analyzing Line \ref{linha:calculo_delta} of Algorithm \ref{algo:fltl} that $\Delta^{i+2} - \Delta^i = \Delta_d$. 
Then

\[\Delta_{y}^i = 2 \cdot r \cdot (1 -cos(\pi / 4)) + \Delta_d / 2 + k \cdot \Delta_d + \Delta_x - \cdots \]
\[ \cdots - 2 \cdot r \cdot sin(\pi/4)\]
\[\Delta_{y}^{i+2} = 2 \cdot r \cdot (1 -cos(\pi / 4)) + \Delta_d / 2 + (k + 1)\cdot \Delta_d + \Delta_x - \cdots \]
\[ \cdots - 2 \cdot r \cdot sin(\pi/4) \]
\[\Delta_{y}^{i+2} - \Delta_{y}^i = \Delta_d\]

If $n$ is odd, it follows, by analyzing Lines \ref{linha:calculo_delta_odd_1} to \ref{linha:calculo_delta_odd_2} 
of Algorithm \ref{algo:fltl}, 
that $\Delta^{i+2} - \Delta^i = \Delta_d$. Then

\[\Delta_{y}^i = 2 \cdot r \cdot (1 -cos(\pi / 4)) + k \cdot \Delta_d + \Delta_x - \cdots \]
\[ \cdots - 2 \cdot r \cdot sin(\pi/4)\]
\[\Delta_{y}^{i+2} = 2 \cdot r \cdot (1 -cos(\pi / 4)) + (k + 1)\cdot \Delta_d + \Delta_x - \cdots \]
\[ \cdots - 2 \cdot r \cdot sin(\pi/4) \]
\[\Delta_{y}^{i+2} - \Delta_{y}^i = \Delta_d\]
and therefore, $\Delta_{y}^{i+2} - \Delta_{y}^i = \Delta_d$ for every $n$ and for every $i$.

Similar reasoning can be applied to $\Delta_x$. Clearly, if $n$ is even and $i$ is even, it can be concluded by 
analyzing Algorithm \ref{algo:advancing} that $\Delta_m^i - \Delta_m^{i+1} = \Delta_d$. 
Then

\[\Delta_{x}^i = k \cdot \Delta_d + 2 \cdot r \cdot sin(\pi/4) + \Delta_y - 2 \cdot r \cdot (1 -cos(\pi / 4))\]
\[\Delta_{x}^{i+1} = (k+1) \cdot \Delta_d - \Delta_d + 2 \cdot r \cdot sin(\pi/4) + \cdots \]
\[ \cdots + \Delta_y - 2 \cdot r \cdot (1 -cos(\pi / 4))\]
\[\Delta_x^{i+1} - \Delta_x^i = 0\]
that is the same case that happens when $n$ is odd and $i$ is even. 
If $n$ is even and $i$ is odd, it follows by analyzing Algorithm \ref{algo:advancing} that 
$\Delta_m^{i+1} - \Delta_m^i = 2 \cdot \Delta_d$, and then

\[\Delta_{x}^i = k \cdot \Delta_d + 2 \cdot r \cdot sin(\pi/4) + \Delta_d + \Delta_y - \cdots \]
\[\cdots - 2 \cdot r \cdot (1 -cos(\pi / 4))\]
\[\Delta_{x}^{i+1} = (k+2) \cdot \Delta_d - \Delta_d + 2 \cdot r \cdot sin(\pi/4) + \cdots \]
\[\cdots + \Delta_y - 2 \cdot r \cdot (1 -cos(\pi / 4))\]
\[\Delta_d^{i+1} - \Delta_d^i = 0\]
that is the same case that happens when $n$ is odd and $i$ is odd. 
Therefore, $\Delta_d^{i+1} - \Delta_d^i = 0$
for every $n$ and for every $i$. Since the same reasoning can be  done  with
different  cooperative  system bases, and since all aircraft  assume the  longitudinal formation at the same time instant,
the  theorem is proved.
\end{proof}

\section{Conclusion}
\label{sec:conclusion}

An algorithm for changing from longitudinal to
latitudinal formation for autonomous aircraft squadrons is presented in this paper.
Despite the relevance of this problem,  an  extensive literature review did not
produce relevant results. 
 
The proposed FLOTLA  algorithm's  time
complexity is optimal, since it is equal to the problem
lower bound  \cite{Cormen}. This result was not different 
from the opposite case considered in \cite{Giacomin-sbai}. 

It was proved that if 
each aircraft is equally spaced from its neighbors
by distance $\Delta_d$ on the longitudinal formation,
then the proposed algorithm makes the
squadron  change  its formation to  a  latitudinal one,
keeping the same distance $\Delta_d$ from each aircraft
and its neighbors. 

The references were created correctly by the proposed algorithm, 
such that they could be followed by a control scheme, and
the theoretical analysis was confirmed by the simulations, by showing that 
the aircraft do not collide during the formation transition.
Additionally, the number of aircraft has
to be limited because the aircraft use different velocities,  which  have to be verified at design time, for
security reasons.

\bibliography{Bibliografia}

\end{document}